\documentclass{article}

\usepackage[preprint]{neurips_2026}

\usepackage[utf8]{inputenc}
\usepackage[T1]{fontenc}
\usepackage{hyperref}
\usepackage{url}
\usepackage{booktabs}
\usepackage{amsfonts}
\usepackage{amssymb}
\usepackage{bm}
\usepackage{nicefrac}
\usepackage{microtype}
\usepackage{xcolor}
\usepackage{algorithm}
\usepackage{algpseudocode}
\usepackage{enumitem}

\usepackage{conference}
\usepackage{mathabx}

\title{EVE-Agent: Evidence-Verifiable Self-Evolving Agents}

\author{Yamato Arai 
\\
Fujitsu Limited \\  
Department of Basic Science\\
The University of Tokyo\\
\And 
Yuma Ichikawa 
\\
Fujitsu Limited \\
RIKEN center for AIP
}

\begin{document}

\maketitle

\begin{abstract}
    Self-evolving agents should not train on examples they cannot justify. Data-free self-evolving search agents offer a scalable route to systems that generate their own questions, answer them, and improve from their own feedback without human annotations. Yet, without verifiable evidence, this loop can reward fluent but unsupported examples, turning the self-generated curriculum into an opaque and potentially unreliable training signal. We argue that evidence verifiability is a prerequisite for trustworthy self-evolution in search agents: each generated instance should include not only an answer but also a source-grounded span whose contribution to that answer can be measured. We introduce EVE-Agent, an Evidence-Verifiable Self-Evolving Agent that operationalizes this principle through a modification to the proposer--solver framework. The proposer generates a question, an answer, and a verbatim evidence span. An evidence verifier then rewards the span according to the marginal accuracy gain when the evidence is provided. This produces a training signal that favors evidence that genuinely helps answer the question, without requiring oracle answers, human labels, or external annotations. EVE-Agent leaves the backbone model, retriever, search tool, and optimization framework unchanged. Experiments show that EVE-Agent substantially improves evidence-grounded correctness over prior self-evolving search agents. The resulting curriculum is not merely self-generated but auditable by construction: each training example carries an inspectable source span that explains why it should be trusted.
\end{abstract}

\section{Introduction}\label{sec:intro}

Search agents for knowledge-intensive question answering must do more than retrieve relevant information: they must also ground their answers in appropriate evidence. This requirement distinguishes them from standard language models, which may generate fluent responses without revealing why those responses should be trusted. In supervised settings, evidence grounding is commonly enforced using human-curated question--answer datasets annotated with supporting spans, such as HotpotQA, 2WikiMultiHopQA, and MuSiQue~\citep{yang2018hotpotqa,ho20202wikimultihopqa,trivedi2022musique}. Retrieval-augmented and tool-using language models operationalize this evidence-seeking behavior~\citep{lewis2020retrieval,yao2023react,schick2023toolformer,trivedi2023interleaving,asai2023selfrag}, while recent reinforcement-learning methods further demonstrate that models can learn to invoke search as part of their reasoning process~\citep{jin2025search,song2025r1searcher}.

However, constructing evidence-grounded supervision at scale is costly, tightly coupled to a particular corpus, and difficult to update when the retrieval environment changes. Data-free self-evolution provides an attractive alternative: a model generates its own training questions, attempts to solve them, and improves from the resulting feedback. This paradigm has shown strong promise in domains such as reasoning and code, where self-generated tasks can be validated by external oracles, including interpreters and symbolic checkers~\citep{zhao2025absolute,huang2026rzero}. It has also recently been extended to multi-turn search agents~\citep{yue2026drzero}. However, search-based question answering lacks the exact verification mechanisms available in code or mathematics. A self-generated question may be ambiguous, unsupported by the source text, or answerable from the model's memorized knowledge alone. Likewise, a solver may produce a confident answer without supplying evidence that genuinely supports it.

\begin{figure}[tb]
    \centering
    \includegraphics[width=0.9\linewidth]{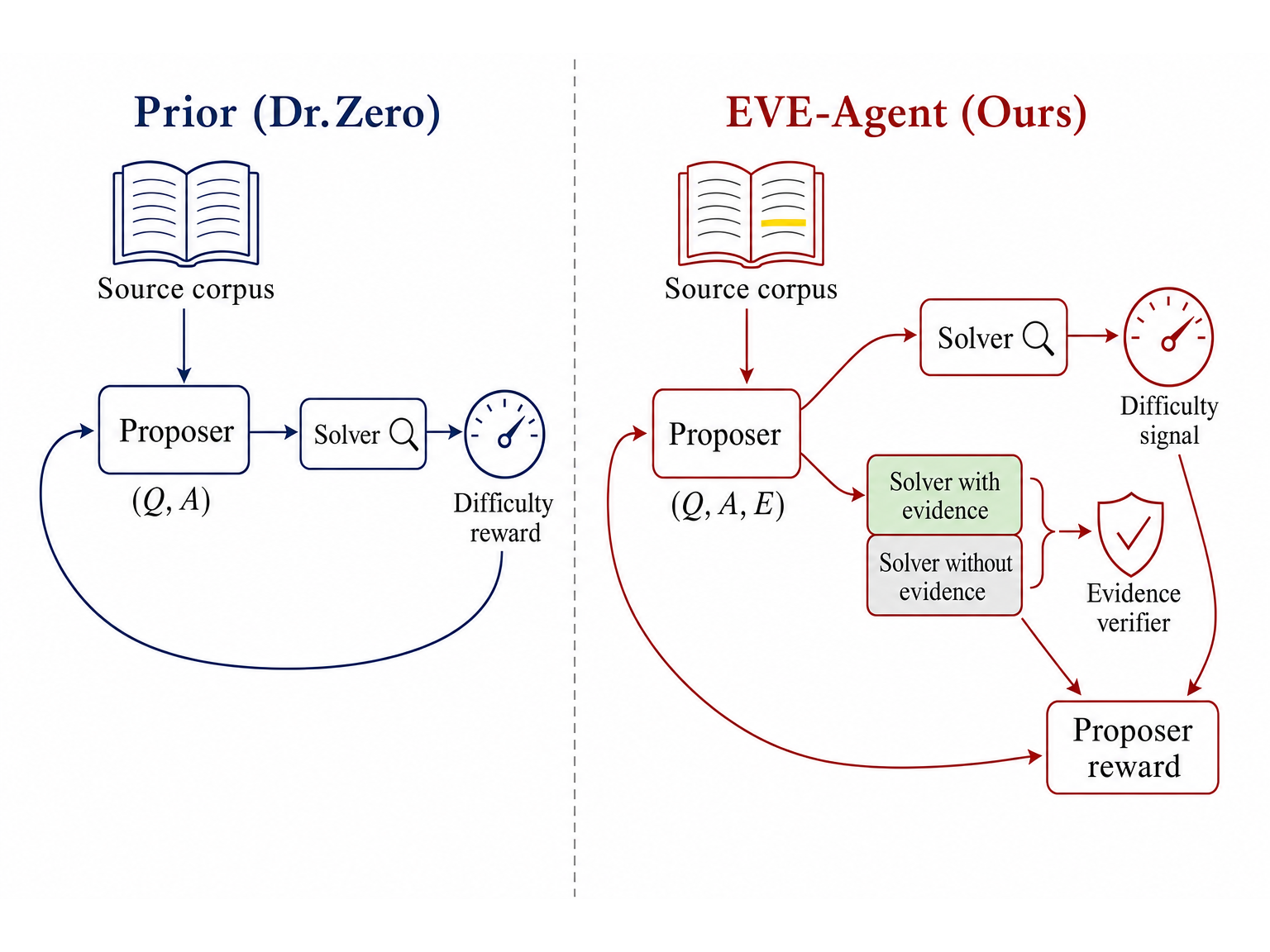}
    \caption{\textbf{Evidence-verifiable self-evolving search agents.} Existing self-evolving search agents (\emph{left}) reward proposers using only a difficulty signal based on solver accuracy, without auditing the source evidence behind each question. EVE-Agent (\emph{right}) requires the proposer to output a source-grounded evidence span and rewards it only when that evidence causally improves the solver's answer accuracy, measured by the gain from no-evidence to with-evidence rollouts. This modification is limited to the reward: the proposer, solver, backbone model, and search tool remain unchanged.}
    \label{fig:loop}
\end{figure}

This limitation exposes a central weakness in existing self-evolving search-agent loops. Their reward signals primarily assess whether a generated question is useful as a difficulty-controlled training instance. However, they do not verify whether the corresponding answer is grounded in a source span that can be checked. Consequently, unsupported examples may enter the self-generated curriculum and shape subsequent learning. The problem is not simply that evidence may be missing. In practice, a system may produce a syntactically valid evidence block even when the cited span does not actually justify the answer. Such examples are difficult to audit: once they are incorporated into the curriculum, it becomes unclear whether the agent is learning to search for and reason over evidence, or merely reinforcing fluent but unverifiable behavior.

We argue that evidence verifiability should be a core design principle for data-free, self-evolving search agents. Each generated training instance should include a source-grounded evidence span, and the utility of that span should be explicitly measurable. This requirement reframes evidence from an optional explanation into a training-time object that can be inspected, scored, and reused. It also makes the generated curriculum more trustworthy: every question--answer pair is paired with a concrete textual basis, and the system can be evaluated not only on whether it answers correctly but also on whether it provides evidence that supports its answer.
To this end, we introduce EVE-Agent, a lightweight extension of the proposer--solver framework built around evidence verifiability. The proposer generates a question, an answer, and a verbatim evidence span from the source text. An evidence verifier then rewards the proposer according to the marginal improvement in the current solver's answer accuracy when the evidence is provided, relative to answering from the question alone. This signal requires no oracle answers, human labels, or external annotations: it is computed solely from the solver, the proposer-emitted evidence, and the corpus. The same evidence span is subsequently used to train the solver to produce both an answer and supporting evidence. Importantly, EVE-Agent leaves the backbone model, retriever, search tool, and policy-optimization framework unchanged.

The resulting self-generated curriculum is auditable by design. Each training example is paired with an explicit source span, and that span is rewarded only when it helps the solver answer the question. This design discourages unsupported or purely memorization-based questions while preserving the scalability advantages of data-free self-evolution. It also offers a practical mechanism for inspecting the agent's generated training data: the curriculum is no longer a collection of opaque question--answer pairs, but a set of evidence-linked instances whose grounding can be verified post hoc.
Our experiments show that, under matched conditions, EVE-Agent substantially improves evidence-grounded correctness over prior self-evolving search-agent methods. These results suggest that self-evolving search agents can be trained not only to answer questions but also to produce evidence that makes their own training process verifiable.

\section{Background}\label{sec:prelim}

\paragraph{Notation.}
Let $\mathcal{D}=\{d_1,\ldots,d_{|\mathcal{D}|}\}$ denote a finite corpus, where each document $d_i$ is represented as a token sequence. A task instance is a triple $(q,a,e)$ consisting of a question $q$, its target answer $a$, and an evidence span $e$. The evidence span is required to be a contiguous text span copied from either a source document in $\mathcal{D}$ or a snippet retrieved from that corpus. A search engine $\mathcal{R}$ is shared by all agents: given a text query, $\mathcal{R}$ returns a finite list of snippets drawn from $\mathcal{D}$. Throughout the paper, logarithms are natural, and $\mathbf{1}\{\cdot\}$ denotes the indicator function, which equals $1$ when its argument is true and $0$ otherwise. For any model $M$ and input $x$, we use $M(a \mid x) \coloneqq \mab{P}_{\hat a \sim M(\cdot \mid x)}[\hat a = a]$
to denote the probability that $M$ generates the answer string $a$ when conditioned on $x$.

\paragraph{Self-evolving search-agent loop.}
The self-evolving search-agent framework consists of two policies that are updated over training rounds $t=1,\ldots,T$. The proposer policy, denoted by $\pi_t^{\mathrm{pro}}$, generates a training instance from a source document. In the prior framework, this instance is a question--answer pair $(q,a)$; in EVE-Agent, it is extended to a question--answer--evidence triple $(q,a,e)$. The solver policy, denoted by $\pi_t^{\mathrm{sol}}$, receives a question, may call the shared search engine $\mathcal{R}$ during its reasoning process, and outputs an answer. For notational convenience, we define
\begin{equation}
    M_{\mathrm{sol},t}(\cdot \mid x)
    \coloneqq
    \pi_t^{\mathrm{sol}}(\cdot \mid x,\mathcal{R})
    \label{eq:solver-induced-distribution}
\end{equation}
for the answer distribution induced by the solver at round $t$ when it is given input $x$ and access to the search engine. We use $M_{\mathrm{pro},t}$ analogously for the distribution induced by the proposer. The setting is data-free in the sense that no human-labeled question--answer pairs or supporting spans are provided. The only human-supplied resources are the corpus $\mathcal{D}$ and the search engine $\mathcal{R}$.

\paragraph{Difficulty reward.}
We first recall the difficulty-based proposer reward used in the prior self-evolving search-agent framework~\citep{yue2026drzero}. At training round $t$, a source document $d\in\mathcal{D}$ is sampled uniformly from the corpus. The proposer then generates a question--answer pair
$(q,a)\sim \pi_t^{\mathrm{pro}}(\cdot \mid d,\mathcal{R})$,
and the solver answers the same question $n$ times independently, $\{\hat a_j\}_{j=1}^{n} \sim M_{\mathrm{sol},t}(\cdot \mid q)$. Let $k\coloneqq \sum_{j=1}^{n}\mathbf{1}\{\hat a_j=a\}$ be the number of solver predictions that exactly match the proposer-provided answer. The proposer receives the difficulty reward
\begin{equation}
    R^{\mathrm{DZ}}_t(q,a;k)
    =
    \mathbf{1}\{0<k<n\}\frac{n-k}{n-1}.
    \label{eq:dz-reward-raw}
\end{equation}
The prior system~\citep{yue2026drzero} also adds a format-related term to this signal; in our notation we keep $R^{\mathrm{DZ}}_t$ for the pure difficulty component and treat the format reward separately in Eq.~\eqref{eq:prop-reward}.

This reward favors questions that are neither trivial nor impossible for the current solver. If all solver predictions are correct or all are incorrect, then the difficulty term is zero. Otherwise, the reward increases as the number of incorrect predictions grows. Thus, among questions that the solver can answer at least sometimes, the proposer is encouraged to generate examples near the solver's current learning frontier.
This intuition can be made precise. Define $p_t^{\mathrm{sol}}(q,a) \coloneqq M_{\mathrm{sol},t}(a\mid q)$, the probability that the solver at round $t$ generates answer string $a$ when conditioned on question $q$. Since the $n$ solver predictions are sampled independently, the number of correct predictions follows $k \sim \mathrm{Bin}\bigl(n,p_t^{\mathrm{sol}}(q,a)\bigr)$. Taking the expectation of Eq.~\eqref{eq:dz-reward-raw} over this binomial randomness yields
\begin{equation}
    \mathbb{E}_{k\sim\mathrm{Bin}(n,p)}
    \bigl[
        R^{\mathrm{DZ}}_t(q,a;k)
    \bigr]
    =
    \phi_n\bigl(p_t^{\mathrm{sol}}(q,a)\bigr),
    \label{eq:dz-phi-expectation}
\end{equation}
where
\begin{equation}
    \phi_n(p)
    \coloneqq
    \frac{n}{n-1}(1-p)\bigl(1-(1-p)^{n-1}\bigr).
    \label{eq:dz-phi}
\end{equation}
Lemma~\ref{lem:phi-n} proves this identity. The function $\phi_n$ is continuous and unimodal on $[0,1]$, and it vanishes at both endpoints. Therefore, the expected difficulty reward is small both when the solver almost never answers correctly and when it almost always answers correctly; it is largest at an intermediate success probability.

\paragraph{Hop-grouped relative policy optimization.}
Optimizing the proposer directly with group-relative policy optimization would be costly in the search-agent setting. A direct implementation would require nested sampling: for each source document, one would sample multiple candidate questions from the proposer, and for each candidate question one would run multiple solver rollouts. Because each solver rollout may call the search engine several times, this nested procedure is prohibitively expensive.

Hop-grouped relative policy optimization (HRPO) avoids this cost by sampling one proposer output per source document while normalizing rewards within comparable groups~\citep{yue2026drzero}. Suppose a batch contains $N$ generated question--answer pairs, $\{(q_i,a_i)\}_{i=1}^{N}$.
Each question has a prescribed hop count $h_i\in\mathcal{H}$, where the hop count indicates the intended number of reasoning steps or evidence pieces needed to answer the question. For each hop value $h$, define $\mathcal{I}_h \coloneqq \{i: h_i=h\}$.

HRPO computes the standardized advantage of example $i$ within its hop group:
\begin{equation}
    A_{i,h}
    =
    \frac{
        R^{\mathrm{DZ}}_t(q_i,a_i;k_i)
        -
        \mathbb{E}_{j\in\mathcal{I}_h}
        \bigl[
            R^{\mathrm{DZ}}_{t,j}
        \bigr]
    }{
        \sqrt{
            \mathrm{Var}_{j\in\mathcal{I}_h}
            \bigl[
                R^{\mathrm{DZ}}_{t,j}
            \bigr]
        }
        +
        \delta_0
    },
    \label{eq:hrpo-adv}
\end{equation}
where $k_i$ is the number of correct solver predictions for example $i$, $R^{\mathrm{DZ}}_{t,j}$ denotes the difficulty reward of example $j$, and $\delta_0>0$ is a numerical stabilizer.

Let $\pi_{\mathrm{ref}}^{\mathrm{pro}}$ be a frozen reference proposer policy. In our experiments, this reference is the proposer initialization at the beginning of Phase A. The HRPO update maximizes
\begin{equation}
    \mathcal{J}^{\mathrm{HRPO}}_t
    =
    \frac{1}{N}
    \sum_{h\in\mathcal{H}}
    \sum_{i\in\mathcal{I}_h}
    \log
    \pi_t^{\mathrm{pro}}(q_i,a_i\mid d_i,\mathcal{R})
    A_{i,h}
    -
    \beta
    \mathbb{E}_{d}
    \left[
        \mathrm{KL}
        \left(
            \pi_t^{\mathrm{pro}}(\cdot\mid d,\mathcal{R})
            \,\middle\|\,
            \pi_{\mathrm{ref}}^{\mathrm{pro}}(\cdot\mid d,\mathcal{R})
        \right)
    \right],
    \label{eq:hrpo-objective}
\end{equation}
where $\beta>0$ controls the strength of KL regularization. The KL divergence is taken between two conditional distributions over proposer outputs given the same source document $d$ and access to the same search engine $\mathcal{R}$. Its role is to keep the current proposer $\pi_t^{\mathrm{pro}}$ close to the frozen reference $\pi_{\mathrm{ref}}^{\mathrm{pro}}$, thereby preventing overly large policy updates. This KL term is conceptually separate from the relative advantage normalization in Eq.~\eqref{eq:hrpo-adv}.

\paragraph{Group-relative policy optimization.}
The solver is trained with group-relative policy optimization (GRPO)~\citep{shao2024deepseekmath}. For a given question $q$, the rollout policy is the previous solver $\pi_{t-1}^{\mathrm{sol}}$. It samples a group of $n$ candidate responses, $\{\hat y_j\}_{j=1}^{n}
    \sim
    \pi_{t-1}^{\mathrm{sol}}(\cdot\mid q,\mathcal{R})$.
Each response receives a binary answer reward $r_j =
    \mathbf{1}\{\hat y_j=a\}$,
where $a$ is the target answer. Let $\overline r$ and $\hat\sigma$ be the empirical mean and standard deviation of $\{r_j\}_{j=1}^{n}$ within the group. The standardized advantage is
$A_j =
    \nicefrac{r_j-\overline r}{\hat\sigma+\delta_0},$
where $\delta_0>0$ prevents division by zero.

Let $\pi_{\mathrm{ref}}^{\mathrm{sol}}$ be a frozen reference solver policy. In our experiments, this reference is the solver initialization at the beginning of Phase B. GRPO maximizes the clipped surrogate
\begin{multline}
     \mathcal{J}^{\mathrm{GRPO}}_t
    =
    \mathbb{E}
    \left[
        \frac{1}{n}
        \sum_{j=1}^{n}
        \min
        \left(
            \rho_j A_j,
            \mathrm{clip}(\rho_j,1-\epsilon,1+\epsilon)A_j
        \right)
    \right] \\
    -
    \beta
    \mathbb{E}_{q}
    \left[
        \mathrm{KL}
        \left(
            \pi_t^{\mathrm{sol}}(\cdot\mid q,\mathcal{R})
            \,\middle\|\,
            \pi_{\mathrm{ref}}^{\mathrm{sol}}(\cdot\mid q,\mathcal{R})
        \right)
    \right],~~\rho_j
    =
    \frac{
        \pi_t^{\mathrm{sol}}(\hat y_j\mid q,\mathcal{R})
    }{
        \pi_{t-1}^{\mathrm{sol}}(\hat y_j\mid q,\mathcal{R})
    },
    \label{eq:grpo}
\end{multline}
where $\epsilon\in(0,1)$ is the clipping width, and $\beta>0$ is the KL coefficient.

The importance ratio $\rho_j$ compares the probability of the sampled response $\hat y_j$ under the current solver $\pi_t^{\mathrm{sol}}$ with its probability under the rollout policy $\pi_{t-1}^{\mathrm{sol}}$. In contrast, the KL term compares the full conditional output distribution of the current solver with the frozen reference solver $\pi_{\mathrm{ref}}^{\mathrm{sol}}$ for the same question $q$ and search engine $\mathcal{R}$. The ratio controls the policy-gradient update on sampled responses, while the KL term regularizes the entire updated policy. EVE-Agent keeps this optimization infrastructure unchanged: it uses HRPO for the proposer and GRPO for the solver, and changes only the reward design and, optionally, the source-document selector.

\section{Method}\label{sec:method}

The difficulty reward of Eq.~\eqref{eq:dz-reward-raw} encourages the proposer whenever the solver is uncertain about a generated question, but it does not verify whether the proposer's answer is supported by any source span: the same reward is paid whether the underlying evidence is genuine or unrelated. Section~\ref{subsec:exp-bottleneck} documents this failure mode empirically; we report it as a motivating diagnostic and devote the remainder of this section to the components that close the gap. Section~\ref{subsec:verifier} introduces the evidence verifier, which scores the proposer-emitted span by its causal effect on the solver's answer accuracy. Section~\ref{subsec:solver} reuses the same span as the supervision target during solver training. Section~\ref{subsec:selector} describes an optional cluster bandit that diversifies the source-document distribution, and Section~\ref{subsec:schedule} explains how the proposer and solver are updated in two sequential phases. Figure~\ref{fig:phaseA} summarizes the resulting Phase A dataflow. The backbone, retriever, search tool, and policy-optimization framework remain unchanged.

\begin{figure}[t]
    \centering
    \includegraphics[width=0.95\linewidth]{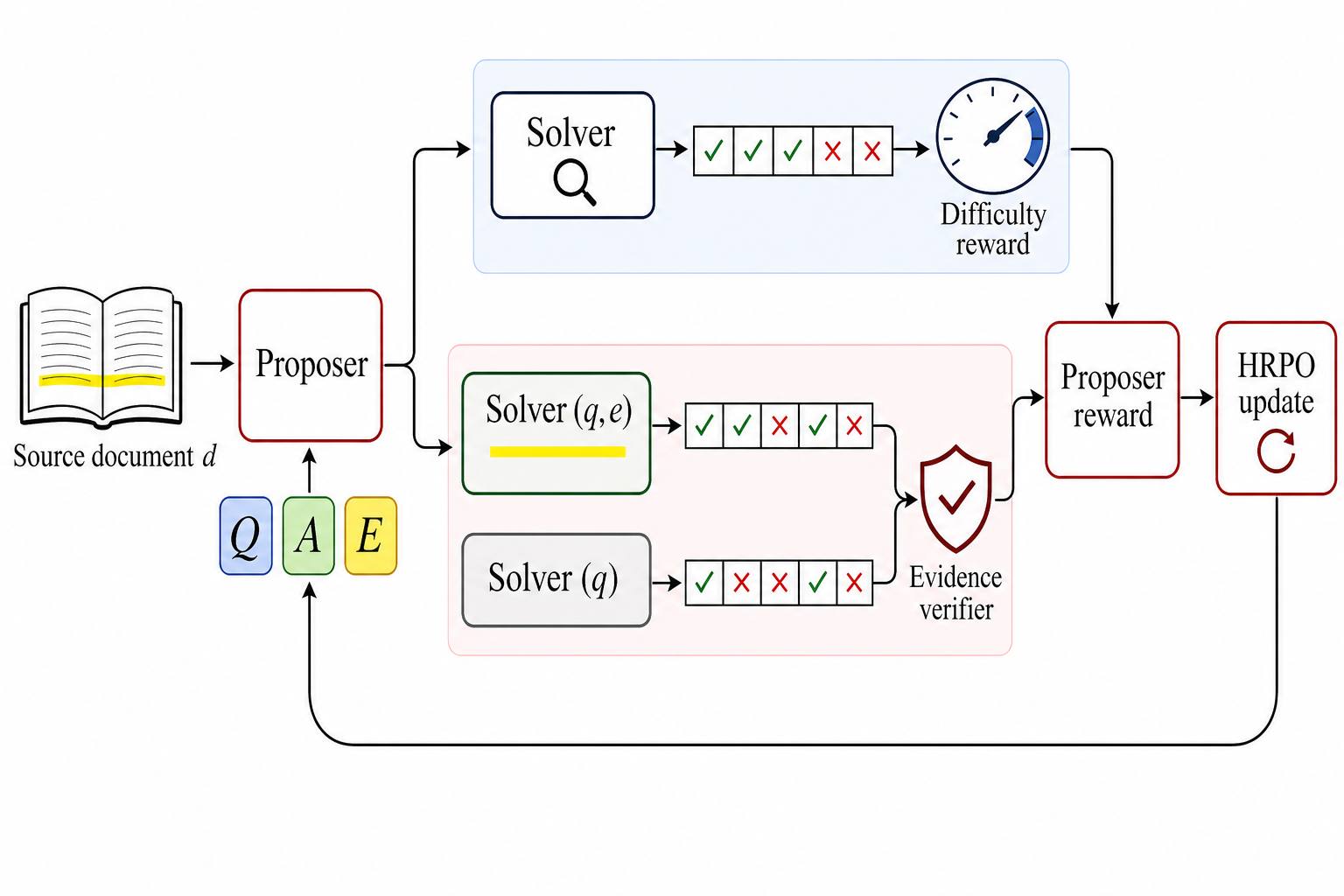}
    \caption{\textbf{One Phase A iteration of EVE-Agent.} The proposer generates a question--answer--evidence triple from the source document $d$. The solver attempts the question with the search tool, producing the difficulty reward of Eq.~\eqref{eq:dz-reward-raw}; in parallel, single-turn search-disabled rollouts of the solver with and without the evidence span produce the evidence verifier of Eq.~\eqref{eq:vqe}. These two signals combine with the format and brevity terms in the proposer reward of Eq.~\eqref{eq:prop-reward}, which drives one HRPO update of the proposer. The backbone, the retriever, and the search tool are unchanged.}
    \label{fig:phaseA}
\end{figure}

\subsection{Evidence verifier for the proposer}\label{subsec:verifier}

\paragraph{Proposer output.}
EVE-Agent changes the proposer output from a question--answer pair to a question--answer--evidence triple. Given a source document $d$ and access to the search engine $\mathcal{R}$, the proposer generates
\begin{equation}
    (q,a,e)\sim \pi_t^{\mathrm{pro}}(\cdot\mid d,\mathcal{R}),
\end{equation}
where $q$ is the generated question, $a$ is the target answer, and $e$ is the evidence span. The evidence span must be copied verbatim from either the source document $d$ or one of the snippets returned by the search engine during the proposer's rollout. This constraint ensures that the evidence is not merely a free-form explanation, but a concrete text span that can be checked against the corpus.

After generation, we parse the output and apply a simple validity filter. Let $\textsc{Norm}(\cdot)$ denote the standard answer-normalization function that lowercases text, removes articles, strips punctuation, and collapses whitespace. A rollout is called \emph{valid} if both $q$ and $a$ are non-empty and the normalized answer is not literally contained in the normalized question: $q\neq \emptyset,~a\neq \emptyset,~
    \textsc{Norm}(a)\not\subseteq \textsc{Norm}(q)$.
The last condition removes degenerate questions that reveal their own answers. Invalid rollouts receive only the format-related reward defined below and are excluded from the evidence verifier.

\paragraph{Format reward.}
Before scoring whether the evidence helps the solver, we first assign a lightweight format reward that checks whether the proposer output is structurally usable. The reward $F_{\mathrm{fmt}}(q,a,e,d,h)\in[0,1]$
depends on the generated question $q$, answer $a$, evidence span $e$, source document $d$, and prescribed hop count $h$. Here, $h$ denotes the intended number of reasoning steps, or equivalently the intended amount of multi-hop search behavior, for the generated question.

The format reward combines four equally weighted components. The first component is an integrity score that is always equal to $1$ for parsed outputs that reach this stage. The remaining three components are sub-scores in $[0,1]$: $F_{\mathrm{think}}$ rewards the presence of an explicit planning step, $F_{\mathrm{tool}}(h)$ rewards the expected number and syntactic correctness of tool calls for an $h$-hop rollout, and $F_{\mathrm{ans}}$ rewards a concise answer that is consistent with the available context. We define
\begin{equation}
    F_{\mathrm{fmt}}(q,a,e,d,h) = \frac{1}{4} \left( 1 + F_{\mathrm{think}} + F_{\mathrm{tool}}(h) + F_{\mathrm{ans}} \right).
    \label{eq:fmt}
\end{equation}
This term is intentionally simple: it ensures that the proposer follows the required output protocol, but it is not meant to judge whether the evidence actually supports the answer. That role is handled by the evidence verifier below. The precise definitions of the three sub-scores are provided in Appendix~\ref{app:format}.

\paragraph{Evidence-quality score.}
We now define the signal that measures whether the proposer-provided evidence is actually useful for answering the generated question. Consider a valid triple $(q,a,e)$, where $q$ is the question, $a$ is the target answer, and $e$ is the evidence span emitted by the proposer. The key idea is to compare two answer probabilities: one when the solver is given both the question and the evidence, and one when it is given the question alone.

Let $\widetilde{\pi}_{\mathrm{sol},t}(\cdot\mid q,e)$ denote the current solver at round $t$ under a single-turn answering protocol: the solver receives the question $q$ and evidence span $e$, but it is not allowed to make any additional search calls. Similarly, let $\widetilde{\pi}_{\mathrm{aux},t}(\cdot\mid q)$ denote an auxiliary scorer under the same single-turn, search-disabled protocol, but conditioned only on the question $q$. We define
\begin{equation}
    \begin{aligned}
        p_{+}^{t}(q,e,a)
        &\coloneqq
        \mab{P}_{\hat a\sim \widetilde{\pi}_{\mathrm{sol},t}(\cdot\mid q,e)}
        \left[
            \hat a=a
        \right], \\
        p_{-}^{t}(q,a)
        &\coloneqq
        \mab{P}_{\hat a\sim \widetilde{\pi}_{\mathrm{aux},t}(\cdot\mid q)}
        \left[
            \hat a=a
        \right].
    \end{aligned}
    \label{eq:pplus-pminus}
\end{equation}
The first quantity, $p_{+}^{t}(q,e,a)$, is the probability of producing the correct answer when the evidence is provided. The second quantity, $p_{-}^{t}(q,a)$, is the probability of producing the same answer without access to that evidence. We define the evidence verifier as their difference:
\begin{equation}
    V_t(q,e,a)
    \coloneqq
    p_{+}^{t}(q,e,a)
    -
    p_{-}^{t}(q,a).
    \label{eq:vqe}
\end{equation}
Since both probabilities lie in $[0,1]$, the verifier score satisfies $V_t(q,e,a)\in[-1,1]$.

This construction isolates the marginal contribution of the provided evidence. Search is disabled in both conditions, so the score does not reward the solver for finding new information after receiving the question. Instead, it asks a narrower and more auditable question: does the span $e$ itself make the answer $a$ easier to produce? A positive verifier score means that conditioning on $e$ increases the solver's probability of generating the target answer. A score near zero means that the evidence provides little additional information beyond the question. A negative score indicates that the evidence makes the target answer less likely, which is consistent with the span being misleading or irrelevant.

In our experiments, the auxiliary scorer uses the same weights as the current solver. Thus, $\widetilde{\pi}_{\mathrm{aux},t}$ and $\widetilde{\pi}_{\mathrm{sol},t}$ differ only in their inputs: the former receives $q$, whereas the latter receives $(q,e)$. Under this choice, $V_t(q,e,a)$ directly measures the conditional answer-accuracy gain induced by the evidence for the current solver. The framework also supports a variant in which the auxiliary scorer is a separately hosted frozen model, but we do not use that configuration in the experiments reported here.

\paragraph{Estimator.}
The verifier score in Eq.~\eqref{eq:vqe} depends on two answer probabilities, which we estimate by Monte Carlo sampling. For a fixed valid triple $(q,a,e)$ and an integer $m\ge 1$, we draw $m$ independent answers from the solver with evidence,
\begin{equation}
    \hat a_{+}^{(j)}
    \sim
    \widetilde{\pi}_{\mathrm{sol},t}(\cdot\mid q,e),
    \qquad
    j=1,\ldots,m,
\end{equation}
and $m$ independent answers from the auxiliary scorer without evidence,
\begin{equation}
    \hat a_{-}^{(j)}
    \sim
    \widetilde{\pi}_{\mathrm{aux},t}(\cdot\mid q),
    \qquad
    j=1,\ldots,m.
\end{equation}
We then estimate the evidence verifier by the difference between the two empirical accuracies:
\begin{equation}
    \widehat V_{t,m}(q,e,a)
    \coloneqq
    \frac{1}{m}
    \sum_{j=1}^{m}
    \mathbf{1}\{\hat a_{+}^{(j)}=a\}
    -
    \frac{1}{m}
    \sum_{j=1}^{m}
    \mathbf{1}\{\hat a_{-}^{(j)}=a\}.
    \label{eq:vhat-mc}
\end{equation}
This estimator is unbiased, and its conditional variance is bounded by $1/(2m)$; both facts are formalized and proved as Proposition~\ref{prop:marginal} in Appendix~\ref{app:theory}. It requires $2m$ additional single-turn decodes per valid proposer rollout---$m$ with evidence and $m$ without---an overhead that is modest because verifier rollouts are search-disabled, whereas the main solver rollouts are multi-turn interactions that may call the search engine several times. We use $m=5$ throughout. A teacher-forced log-probability variant of the verifier is implemented but not used in the reported experiments because the sampling-based estimator is already sufficiently efficient.

\paragraph{Brevity bonus.}
The evidence span should be informative but not unnecessarily long. If the proposer copies a very long passage, the verifier becomes less useful because the span may contain many irrelevant facts or reveal the answer without identifying the specific supporting context. Conversely, an extremely short span may collapse to answer leakage rather than evidence. To encourage concise and targeted evidence, we introduce a brevity bonus. Let $|e|$ denote the number of tokens in the evidence span under the proposer tokenizer. We define
\begin{equation}
    B(e)
    \coloneqq
    \max
    \left(
        0,
        1-\frac{|e|}{L_{\max}}
    \right),
    \qquad
    L_{\max}=256.
    \label{eq:brevity}
\end{equation}
The bonus is largest for short spans and decreases linearly with length. It becomes zero once the evidence span reaches $L_{\max}$ tokens. Thus, this term discourages copy-everything behavior while still allowing enough context for multi-hop or entity-rich questions.

\paragraph{Proposer reward.}
The final proposer reward combines structural validity, question difficulty, evidence usefulness, and evidence brevity. For a rollout with source document $d$, prescribed hop count $h$, generated triple $(q,a,e)$, and $k$ correct solver predictions among $n$ trials, we define
\begin{equation}
    R^{\mathrm{pro}}_t(q,e,a;d,h,k)
    =
    \frac{1}{2}
    F_{\mathrm{fmt}}(q,a,e,d,h)
    +
    R^{\mathrm{DZ}}_t(q,a;k)
    +
    \lambda_V V_t(q,e,a)
    +
    \lambda_B B(e),
    \label{eq:prop-reward}
\end{equation}
where $\lambda_V\ge 0$ controls the strength of the evidence-verifiability term and $\lambda_B\ge 0$ controls the strength of the brevity bonus. In all experiments, we set
\begin{equation}
    (\lambda_V,\lambda_B)=(0.5,0.1).
\end{equation}

Each term has a distinct role. The format reward ensures that the proposer follows the required output protocol. The difficulty reward encourages questions near the solver's learning frontier. The verifier reward favors evidence spans that causally improve the solver's ability to produce the target answer. The brevity bonus discourages overly long evidence spans that would make the verifier less diagnostic. The prior self-evolving search agent~\citep{yue2026drzero} is recovered by setting
\begin{equation}
    (\lambda_V,\lambda_B)=(0,0),
\end{equation}
while keeping the same difficulty and format components. We optimize the proposer with the HRPO objective in Eq.~\eqref{eq:hrpo-objective}, replacing the original difficulty-only reward with $R^{\mathrm{pro}}_t$ and making no other changes to the optimization procedure.

\subsection{Solver reward}\label{subsec:solver}

After training the proposer with the reward in Eq.~\eqref{eq:prop-reward}, we freeze it and use it to construct the solver-training data. Specifically, we roll out the trained proposer over the corpus. Each valid rollout produces a triple $(q,a,e)$, where $q$ is the generated question, $a$ is the target answer, and $e$ is the proposer-provided evidence span. We then treat $e$ as the gold evidence for training the solver. This choice is important: the same evidence span that the proposer was rewarded for producing is reused as the supervision target for the solver.

The solver is trained with the GRPO objective in Eq.~\eqref{eq:grpo}. For each generated training instance, the solver is required to output both an answer $\hat a$ and an evidence span $\hat e$. We define the solver reward as
\begin{equation}
    R^{\mathrm{sol}}(\hat a,\hat e;a,e)
    =
    R_{\mathrm{correct}}(\hat a,a)
    +
    \lambda_E R_{\mathrm{evidence}}(\hat e,e),
    \label{eq:sol-reward}
\end{equation}
where $R_{\mathrm{correct}}$ evaluates answer correctness and $R_{\mathrm{evidence}}$ evaluates evidence recovery. The answer reward is an exact-match indicator after standard answer normalization:
\begin{equation}
    R_{\mathrm{correct}}(\hat a,a)
    =
    \mathbf{1}\{\textsc{EM}(\hat a,a)\}.
    \label{eq:answer-reward}
\end{equation}
Here, $\textsc{EM}(\hat a,a)$ equals true when the normalized predicted answer $\hat a$ exactly matches the normalized target answer $a$.

The evidence reward measures how well the solver's extracted evidence span $\hat e$ matches the proposer-provided evidence span $e$. We use the SQuAD-style token-level F1 score:
\begin{equation}
    R_{\mathrm{evidence}}(\hat e,e)
    =
    \mathrm{F1}_{\mathrm{tok}}
    \left(
        \textsc{Norm}(\hat e),
        \textsc{Norm}(e)
    \right),
    \label{eq:revidence}
\end{equation}
where $\textsc{Norm}$ is the normalization function defined earlier and $\mathrm{F1}_{\mathrm{tok}}$ is the harmonic mean of token-level precision and recall between the normalized predicted and target evidence spans. We set $\lambda_E=0.3$ in all experiments. This reward encourages the solver not only to answer correctly, but also to recover the evidence that supports the answer.

\subsection{Optional corpus selector}\label{subsec:selector}

The default self-evolution loop samples source documents uniformly from the corpus. While simple, uniform sampling can concentrate training on a narrow set of documents or question patterns that happen to receive high reward early in training. This may reduce curriculum diversity: the proposer can repeatedly generate similar questions that lie near the solver's current learning frontier, while underusing other topics and reasoning types.

To mitigate this issue, we include an optional corpus selector that can replace uniform document sampling. The selector is not required for the core evidence-verification mechanism, but it provides a way to diversify the self-generated curriculum. Its goal is to balance two forms of diversity. The first is topic diversity, controlled by clustering documents in embedding space. The second is question-type diversity, controlled by a small set of open-domain QA categories.

Concretely, a frozen sentence encoder maps each document $d\in\mathcal{D}$ to a fixed vector representation. At training round $t$, these document embeddings are partitioned into $K_t$ clusters, where
\begin{equation}
    K_t
    =
    K_0+\lfloor \alpha t\rfloor.
    \label{eq:selector-cluster-schedule}
\end{equation}
Here, $K_0$ is the initial number of clusters, $\alpha\ge 0$ controls how quickly the clustering granularity increases, and $\lfloor\cdot\rfloor$ denotes the floor function. Larger values of $K_t$ correspond to a finer partition of the corpus.

The selector uses two UCB1 bandits~\citep{auer2002finite}. The first is a cluster bandit whose arms correspond to the $K_t$ document clusters. The second is a task-type bandit whose arms correspond to five question types: \textsc{factual}, \textsc{comparison}, \textsc{causal}, \textsc{temporal}, and \textsc{aggregation}. At each sampling step, the cluster bandit selects a topic cluster and the task-type bandit selects a desired question type. A source document is then sampled from the selected cluster with probability proportional to two factors: its distance from the cluster centroid and the inverse of its previous usage count. This favors documents that are both less typical within the cluster and less frequently sampled, encouraging coverage of boundary cases and reducing repeated use of the same documents.

The bandits are updated offline between self-evolution iterations. Their feedback reward combines a diversity term, which penalizes repeatedly selecting already overused clusters or task types, with a utility term, which compares the reward of the current generated sample to the running average reward for the corresponding arm. In this way, the selector encourages exploration without ignoring which parts of the corpus are currently useful for training. Since the selector is orthogonal to the evidence verifier, we do not isolate its empirical effect in the present experiments. The full algorithmic specification is given in Appendix~\ref{app:selector-full}.

\subsection{Two-phase training schedule}\label{subsec:schedule}

We train the proposer and the solver in two sequential phases rather than updating them jointly. This design has two motivations. First, it reduces compute: only one model is optimized with policy gradients at a time. Second, it improves stability. The evidence verifier depends on the solver, so updating the solver while simultaneously using it to score proposer outputs would make the reward landscape non-stationary and difficult to interpret.

In \emph{Phase A}, the solver is kept fixed at its initialization, and the proposer is trained with HRPO using the reward $R^{\mathrm{pro}}_{t}$ in Eq.~\eqref{eq:prop-reward}. The auxiliary scorer used in the verifier shares weights with this fixed solver. For each valid proposer rollout, the verifier score is estimated by the Monte Carlo estimator $\widehat V_{t,m}$ in Eq.~\eqref{eq:vhat-mc}; we use $m=5$ samples. Thus, Phase A teaches the proposer to generate not only difficult questions, but also evidence spans that improve the fixed solver's ability to recover the target answer.

At the end of Phase A, we freeze the trained proposer and use it to generate the solver-training set. We roll it out over a held-out shard of the corpus with the full multi-turn search tool enabled. Each valid rollout contributes one triple $(q,a,e)$, where $q$ is the generated question, $a$ is the target answer, and $e$ is the proposer-provided evidence span. We store $e$ as the gold evidence for the corresponding question--answer pair.

In \emph{Phase B}, the proposer is frozen, and the solver is trained with GRPO using the reward $R^{\mathrm{sol}}$ in Eq.~\eqref{eq:sol-reward}. The solver is required to output both an answer and an evidence span, and it is rewarded for both answer correctness and evidence recovery. This phase transfers the evidence-verifiable curriculum produced by the proposer into the solver.

Relative to a standard self-evolving search agent, the modification is intentionally local. EVE-Agent keeps the backbone model, retriever, search tool, hop grouping, and the HRPO and GRPO optimization objectives unchanged. The core change is the reward design: the proposer is rewarded for generating useful evidence, and the solver is rewarded for recovering that evidence. Two basic guarantees of the reward design---a closed form and a unique interior maximizer for the inherited difficulty reward, and an unbiased, bounded-variance interpretation of the evidence verifier as a marginal answer-accuracy gain---are stated and proved in Appendix~\ref{app:theory}, and Algorithms~\ref{alg:phasea}--\ref{alg:phaseb} in Appendix~\ref{app:pseudocode} give the full data flow.

\section{Experiments}\label{sec:exp}

Our experimental study is organized around two questions, addressed in turn. First, does the difficulty-only reward of prior self-evolving search agents leave a measurable evidence-grounding gap? Section~\ref{subsec:exp-bottleneck} answers this question and isolates the empirical motivation for the verifier. Second, does the evidence verifier of Eq.~\eqref{eq:vqe} close this gap while preserving---or improving---answer accuracy across benchmarks? Section~\ref{subsec:main-results} answers this question and isolates the verifier's contribution under matched compute and matched search tools. Throughout the section we report per-benchmark numbers in tables and discuss the qualitative picture in prose; full quantitative details and additional protocol-specific information are deferred to Appendix~\ref{app:impl-details} and Appendix~\ref{app:bottleneck-full}.

\subsection{Experimental setup}\label{subsec:setup}

\paragraph{Models and search tool.}
The proposer, the solver, and the auxiliary scorer in Eq.~\eqref{eq:vqe} share the Qwen2.5-3B-Instruct~\citep{qwen25} backbone, and the auxiliary scorer reuses the current solver weights. All systems share the same retrieval pipeline: passages from the FlashRAG Wikipedia-2018 snapshot are encoded by E5-base-v2 and indexed by FAISS-IVF; each search call returns the top three passages, and a multi-turn rollout permits at most five assistant turns. Full indexing parameters and decoding settings are listed in Appendix~\ref{app:impl-details}.

\paragraph{Training schedule and key hyperparameters.}
Both phases use a single $8\times$B200 node, a global batch size of $256$, and run for $50$ policy-gradient steps. The proposer-side coefficients $(\lambda_{V},\lambda_{B})$ that enter the proposer reward of Eq.~\eqref{eq:prop-reward} are set to $(0.5,0.1)$, and the brevity cutoff is $L_{\max}=256$ tokens; the verifier Monte Carlo budget in Eq.~\eqref{eq:vhat-mc} is $m=5$. The solver-side coefficient $\lambda_{E}$ in Eq.~\eqref{eq:sol-reward} is set to $0.3$. Hop counts $h\in\{1,2,3,4\}$ are sampled in ratio $4{:}3{:}2{:}1$. Because the auxiliary scorer shares its weights with the solver, the verifier adds only $2m=10$ single-turn, search-disabled decodes per valid proposer rollout, which is negligible compared with the multi-turn search rollouts. The full list of optimizer settings and additional implementation choices is given in Appendix~\ref{app:impl-details}, with hyperparameters summarized in Appendix~\ref{app:hp}.

\paragraph{Benchmarks and metrics.}
We evaluate on seven open-domain QA datasets: NaturalQuestions~\citep{kwiatkowski2019natural}, TriviaQA~\citep{joshi2017triviaqa}, PopQA~\citep{mallen2023popqa}, HotpotQA~\citep{yang2018hotpotqa}, 2WikiMultiHopQA~\citep{ho20202wikimultihopqa}, MuSiQue~\citep{trivedi2022musique}, and Bamboogle~\citep{press2023measuring}. Each system is required to emit an answer and a supporting evidence span. We report three complementary metrics: answer exact match (EM) after answer normalization; an evidence score judged by GPT-4.1 from the question, the gold answer, and the emitted span; and the joint rate at which the answer is correct \emph{and} the evidence is judged supporting. The average column is the unweighted mean over the seven benchmarks.

\paragraph{Compared systems.}
We compare four systems trained or evaluated under matched protocols. \emph{Initial (no search)} and \emph{Initial (search)} are the untrained backbone evaluated without and with the search tool, respectively. \emph{Dr.~Zero} is a faithful re-implementation of~\citet{yue2026drzero} under the same backbone, retrieval corpus, tool, rollout budget, and wall-clock budget. \emph{EVE-Agent} adds the evidence-verifier reward and the matching solver evidence term while leaving the proposer, solver, backbone, and search tool otherwise identical to Dr.~Zero. This design isolates the contribution of evidence-aware reward shaping.

\subsection{Evidence-grounding bottleneck of prior systems}\label{subsec:exp-bottleneck}

\paragraph{Objective.}
We first quantify the failure mode that motivates EVE-Agent: even when a self-evolving search agent attains competitive answer accuracy, the spans it emits as evidence may bear no real relation to the answer. The goal of this experiment is to verify that the difficulty-only reward of prior systems leaves a measurable gap between \emph{producing} an answer and \emph{justifying} it.

\paragraph{Protocol-specific details.}
On top of the shared setup of Section~\ref{subsec:setup}, we sample $1{,}000$ test instances per benchmark, except for Bamboogle, for which we use the full $125$-instance split. Each instance is decoded once by every system with the search tool enabled, and the emitted evidence span is passed to the external judge. We focus on a representative subset of four datasets that spans both single-hop and multi-hop regimes; the full breakdown is given in Appendix~\ref{app:bottleneck-full}.

\paragraph{Results.}
Table~\ref{tab:bottleneck} reports the judged evidence score and the joint answer-and-evidence rate on the four-dataset subset; the corresponding answer-accuracy numbers are given in Appendix~\ref{app:bottleneck-full}. Two patterns are immediate. First, the prior self-evolving search agent does not narrow the evidence-quality gap relative to the untrained backbone, even though its answer accuracy is markedly higher: on the open-domain datasets the judge accepts only a similar fraction of its spans as supporting. Second, the joint correctness rate, which credits an instance only when the answer is right \emph{and} the evidence is supporting, is correspondingly low for the prior system---comparable to the untrained baselines on most benchmarks. EVE-Agent improves both quantities on every dataset in this subset except 2WikiMultiHopQA, where the evidence score is competitive but not the best.

\begin{table}[t]
    \centering
    \caption{Evidence-grounding diagnostics on a representative subset of benchmarks ($1{,}000$ test instances each, except Bamboogle with $125$). \emph{Prior} is a faithful re-implementation of~\citet{yue2026drzero}. All systems use the same backbone and search tool; the external judge sees only the model-emitted evidence span. The full breakdown is given in Appendix~\ref{app:bottleneck-full}.}
    \label{tab:bottleneck}
    \setlength{\tabcolsep}{4pt}
    \footnotesize
    \begin{tabular}{l c c c c c c c}
        \toprule
        & \multicolumn{4}{c}{Evidence score (judge)} & \multicolumn{3}{c}{Answer-and-evidence} \\
        \cmidrule(lr){2-5}\cmidrule(lr){6-8}
        Method & NQ & TriviaQA & HotpotQA & 2Wiki & NQ & TriviaQA & HotpotQA \\
        \midrule
        Initial (no search)             & $0.274$ & $0.424$ & $0.247$ & $0.199$ & $0.052$ & $0.180$ & $0.048$ \\
        Initial (search)                & $0.374$ & $0.373$ & $0.199$ & $0.173$ & $0.024$ & $0.093$ & $0.023$ \\
        Prior~\citep{yue2026drzero}     & $0.242$ & $0.289$ & $0.209$ & $\mathbf{0.205}$ & $0.021$ & $0.098$ & $0.035$ \\
        \rowcolor{blue!5}
        \textbf{EVE-Agent}              & $\mathbf{0.484}$ & $\mathbf{0.582}$ & $\mathbf{0.332}$ & $0.166$ & $\mathbf{0.242}$ & $\mathbf{0.342}$ & $\mathbf{0.152}$ \\
        \bottomrule
    \end{tabular}
\end{table}

\paragraph{Discussion.}
The bottleneck is not that prior systems omit evidence: Appendix~\ref{app:bottleneck-full} shows that they emit a syntactically valid evidence block in over $90\%$ of rollouts, comparable to the initial backbone. The bottleneck is that the emitted block frequently fails to justify the predicted answer, which is consistent with the difficulty reward's design: it audits whether a question is challenging, not whether it is supported. The evidence verifier of Section~\ref{subsec:verifier} addresses this by rewarding spans that causally raise the solver's probability of producing the target answer.

\subsection{Main results across benchmarks}\label{subsec:main-results}

\paragraph{Objective.}
Having identified the bottleneck, we now ask whether the evidence verifier closes it across the full benchmark suite without sacrificing answer accuracy. We compare EVE-Agent against the matched Dr.~Zero re-implementation and the two untrained baselines on all seven datasets and report each of the three metrics in turn (answer EM, evidence score, joint correctness).

\paragraph{Protocol-specific details.}
We reuse the setup of Section~\ref{subsec:setup} verbatim; the only change relative to Section~\ref{subsec:exp-bottleneck} is that all seven benchmarks are now included. Decoding is greedy and the search tool is enabled for both the prior system and EVE-Agent. The external judge configuration is unchanged.

\paragraph{Answer accuracy.}
Answer exact match is the standard summary metric for open-domain QA; we report it first because a verifier that improved evidence quality at the cost of answer correctness would not be useful. Table~\ref{tab:main-answer} reports answer EM. EVE-Agent attains the best average EM and is strongest on five of the seven benchmarks---NQ, TriviaQA, PopQA, HotpotQA, and 2WikiMultiHopQA---each of which has a meaningfully large evaluation pool. On MuSiQue and Bamboogle the picture is mixed: Dr.~Zero edges out EVE-Agent on MuSiQue by a small margin, and the untrained no-search backbone is highest on the small Bamboogle split. Overall, the evidence-oriented reward does not trade off answer correctness for explanation format; on the contrary, accuracy improves under the same compute and search budget as the prior self-evolving search agent.

\begin{table*}[t]
    \centering
    \caption{Answer accuracy (exact match) on seven open-domain QA benchmarks.}
    \label{tab:main-answer}
    \setlength{\tabcolsep}{4pt}
    \renewcommand{\arraystretch}{0.95}
    \footnotesize
    \begin{tabular}{l c c c c c c c c}
        \toprule
        Method & NQ & TriviaQA & PopQA & HotpotQA & 2WikiMQA & MuSiQue & Bamboogle & Avg \\
        \midrule
        Initial (no search) & 0.068 & 0.219 & 0.081 & 0.057 & 0.034 & 0.015 & \textbf{0.200} & 0.096 \\
        Initial (search)    & 0.032 & 0.125 & 0.055 & 0.029 & 0.013 & 0.010 & 0.024          & 0.041 \\
        Dr.~Zero            & 0.069 & 0.257 & 0.134 & 0.110 & 0.077 & \textbf{0.055} & 0.104  & 0.115 \\
        \rowcolor{blue!5}
        EVE-Agent           & \textbf{0.289} & \textbf{0.437} & \textbf{0.300} & \textbf{0.209} & \textbf{0.176} & 0.050 & 0.088 & \textbf{0.221} \\
        \bottomrule
    \end{tabular}
\end{table*}

\paragraph{Evidence quality.}
Evidence quality is the central diagnostic of the bottleneck of Section~\ref{subsec:exp-bottleneck}: only this metric reveals whether the predicted answer is paired with a source span that a careful reader could use to verify it. Table~\ref{tab:main-evidence} reports the judged evidence score across all seven benchmarks. EVE-Agent substantially improves evidence support on all single-hop benchmarks---NQ, TriviaQA, and PopQA---and on HotpotQA, and the average evidence score across the seven benchmarks is the highest of the four systems by a clear margin. The remaining benchmarks are mixed: on 2WikiMultiHopQA Dr.~Zero is slightly stronger; on MuSiQue the search-equipped initial backbone narrowly edges out EVE-Agent; and on Bamboogle the small evaluation split favors the no-search backbone. These exceptions notwithstanding, the verifier-shaped reward produces evidence that the external judge accepts as supporting more often than any baseline, including the search-equipped initial backbone. Crucially, this improvement is obtained with the same backbone, retriever, and search tool used by Dr.~Zero, so it can be attributed to the reward design rather than to additional capacity or retrieval signal.

\begin{table*}[t]
    \centering
    \caption{Evidence score judged by an external GPT-4.1 evaluator. The judge sees only the question, the gold answer, and the model-emitted evidence span.}
    \label{tab:main-evidence}
    \setlength{\tabcolsep}{4pt}
    \renewcommand{\arraystretch}{0.95}
    \footnotesize
    \begin{tabular}{l c c c c c c c c}
        \toprule
        Method & NQ & TriviaQA & PopQA & HotpotQA & 2WikiMQA & MuSiQue & Bamboogle & Avg \\
        \midrule
        Initial (no search) & 0.274 & 0.424 & 0.153 & 0.247 & 0.199 & 0.093 & \textbf{0.376} & 0.252 \\
        Initial (search)    & 0.374 & 0.373 & 0.298 & 0.199 & 0.173 & \textbf{0.103} & 0.176 & 0.242 \\
        Dr.~Zero            & 0.242 & 0.289 & 0.208 & 0.209 & \textbf{0.205} & 0.086 & 0.128 & 0.195 \\
        \rowcolor{blue!5}
        EVE-Agent           & \textbf{0.484} & \textbf{0.582} & \textbf{0.392} & \textbf{0.332} & 0.166 & 0.101 & 0.136 & \textbf{0.313} \\
        \bottomrule
    \end{tabular}
\end{table*}

\paragraph{Joint answer-and-evidence correctness.}
The joint metric is the strictest of the three: an instance is credited only when the predicted answer matches the gold answer \emph{and} the emitted span is judged supporting. From the perspective of evidence verifiability, this is the regime that matters most, because it isolates the cases in which the agent's output is simultaneously correct and auditable. Table~\ref{tab:main-joint} reports the result across all seven benchmarks. EVE-Agent obtains the best average score and is strongest on six of seven benchmarks; on the average column it improves over the matched Dr.~Zero re-implementation by a wide margin. The improvement is not a redistribution between the previous two metrics: EVE-Agent's answer-accuracy gains are concentrated in instances whose evidence is also judged sufficient, which is precisely the population that can be reused as reliable training signal for downstream learning or as a basis for human inspection of the curriculum.

\begin{table*}[t]
    \centering
    \caption{Joint answer-and-evidence score. An instance is counted only when the answer is correct \emph{and} the emitted evidence is judged supporting by the external judge.}
    \label{tab:main-joint}
    \setlength{\tabcolsep}{4pt}
    \renewcommand{\arraystretch}{0.95}
    \footnotesize
    \begin{tabular}{l c c c c c c c c}
        \toprule
        Method & NQ & TriviaQA & PopQA & HotpotQA & 2WikiMQA & MuSiQue & Bamboogle & Avg \\
        \midrule
        Initial (no search) & 0.052 & 0.180 & 0.069 & 0.048 & 0.016 & 0.015 & \textbf{0.184} & 0.081 \\
        Initial (search)    & 0.024 & 0.093 & 0.042 & 0.023 & 0.008 & 0.008 & 0.024          & 0.032 \\
        Dr.~Zero            & 0.021 & 0.098 & 0.059 & 0.035 & 0.035 & 0.023 & 0.040          & 0.044 \\
        \rowcolor{blue!5}
        EVE-Agent           & \textbf{0.242} & \textbf{0.342} & \textbf{0.264} & \textbf{0.152} & \textbf{0.070} & \textbf{0.038} & 0.064 & \textbf{0.167} \\
        \bottomrule
    \end{tabular}
\end{table*}

\paragraph{Discussion.}
Read together, Tables~\ref{tab:main-answer}--\ref{tab:main-joint} support the central claim of the paper. The evidence verifier of Eq.~\eqref{eq:vqe} produces a training-time signal whose gains are not bought by relaxing answer correctness, by overfitting to evidence presence, or by exploiting additional compute relative to the prior system. The answer accuracy and the evidence quality improve together, and the strict joint metric---the metric most aligned with evidence verifiability---exhibits the largest relative gap to prior work. The pattern is consistent across single-hop and multi-hop benchmarks, with the two exceptions involving either a small evaluation pool (Bamboogle, $125$ instances) or a regime in which the prior system already matches the untrained backbone on evidence grounding (2WikiMultiHopQA). The persistence of the improvement on the strict joint metric is, in our view, the cleanest summary of the verifier's effect: under matched compute and matched search tools, EVE-Agent generates curricula and trains solvers whose answers are simultaneously more often correct and more often verifiable.

\section{Related work}\label{sec:related}

\paragraph{Data-free self-evolving reasoning and search agents.}
Absolute Zero~\citep{zhao2025absolute} introduced the data-free
paradigm with a Python-interpreter oracle and R-Zero~\citep{huang2026rzero}
generalized it through a challenger--solver decoupling. Recent
work~\citep{yue2026drzero} ports the loop to multi-turn search agents
with hop-grouped relative policy optimization; SAGE~\citep{sage2026}
adds multi-agent critics and AReaL-SEA~\citep{areal2026} adds
multi-turn tool use to the same template. EVE-Agent differs
in injecting a data-free \emph{evidence verifier} so that the
proposer's reward depends on whether the emitted span causally helps
the trained solver, not only on whether the solver is uncertain.

\paragraph{Verifier-based and retrieval-augmented RL.}
Verifier-based rewards~\citep{lambert2024rlvr,shao2024deepseekmath,cobbe2021training}
and knowledge-graph verifiers~\citep{k2v2026} either assume an
external oracle or pay a heavy graph-construction cost; in contrast,
the verifier of Eq.~\eqref{eq:vqe} is defined entirely through the
trained solver, the proposer-emitted evidence, and the corpus.
Search-R1~\citep{jin2025search} and
R1-Searcher~\citep{song2025r1searcher} are retrieval-augmented RL
agents trained on supervised question--answer pairs, and Self-RAG and
IRCoT provide self-critic and interleaved retrieval
templates~\citep{asai2023selfrag,trivedi2023interleaving} that we
inherit at the proposer level.

\paragraph{Curriculum diversity.}
Semantic diversity rewards~\citep{wan2026dsdr} and the R-Diverse
template of~\citet{li2026rdiverse} act \emph{after} sampling by
down-weighting near-duplicates. The optional selector of
Section~\ref{subsec:selector} instead acts \emph{before} sampling via
a cluster bandit, in the lineage of UCB1~\citep{auer2002finite} and
non-stationary bandit
analyses~\citep{garivier2011onupper,lattimore2020bandit};
curriculum-learning methods~\citep{graves2017automated,matiisen2020teacher}
provide background on schedule design more broadly.

\paragraph{Evidence-grounded evaluation of search agents.}
Evidence-aware evaluation benchmarks~\citep{neoqa2025} formalize the
question of whether a model emits a supporting span; our verifier
turns this diagnostic into a training-time signal.

\section{Conclusion}\label{sec:conclusion}

We argued that \emph{evidence verifiability} should be treated as a first-class property of data-free self-evolving search agents. In existing systems, the proposer is rewarded for generating difficult questions, but no reward checks whether the predicted answer is grounded in a span that a careful reader could use to verify it. We documented the resulting bottleneck empirically on representative open-domain QA benchmarks: a faithful re-implementation of the prior self-evolving search agent answers questions competitively yet emits evidence whose judged quality is no better than that of an untrained backbone.

EVE-Agent closes this gap with a minimal extension that is local to the reward. The proposer is required to emit a verbatim source span together with its question and answer, and the span is credited only when its inclusion causally improves the current solver's ability to recover the target answer; the solver is in turn trained to reproduce both the answer and the supporting span. The backbone, the retriever, the search tool, and the policy-optimization framework remain unchanged.

Across seven open-domain QA benchmarks, under matched compute and matched search tools, this single change yields curricula and solvers whose outputs are simultaneously more often correct and more often verifiable: EVE-Agent improves on answer accuracy, on judged evidence quality, and---most importantly---on the strict joint metric that credits an instance only when both are adequate. The resulting training loop is auditable by construction. Every generated example carries an inspectable source span whose contribution can be checked against the current solver, and the curriculum can be reviewed instance by instance rather than treated as an opaque collection of question--answer pairs.

We see this as a small but concrete step toward \emph{safer self-evolution}. As data-free agents increasingly write their own training data, evidence grounding becomes a prerequisite for trust: an agent's improvements should not rely on training signal that cannot itself be verified, and the soundness of the learning process should be checkable from the data the agent produces, not only from the metrics it eventually reports. EVE-Agent treats evidence as a training-time object that can be inspected, scored, and reused; we expect that incorporating similar verifiability constraints will become standard practice as self-evolving search agents are deployed in settings where their training data, and not only their final answers, must be trusted.

\clearpage

\bibliographystyle{plainnat}
\bibliography{ref}

\appendix
\newpage

\section{Notation}\label{app:notation}

Table~\ref{tab:symbols} collects the principal symbols used in the
main text. Throughout the appendix, log-probabilities are natural
logarithms, indicator functions take values in $\{0,1\}$, and KL
divergences are non-negative and can be infinite.

\begin{table}[h]
    \centering
    \caption{Glossary of the principal symbols used throughout the paper.}
    \label{tab:symbols}
    \renewcommand{\arraystretch}{1.15}
    \setlength{\tabcolsep}{5pt}
    \small
    \begin{tabular}{@{}l l l@{}}
        \toprule
        \textbf{Symbol} & \textbf{Meaning} & \textbf{First used} \\
        \midrule
        $\mac{D}$                                 & source corpus                                                                              & Section~\ref{sec:prelim} \\
        $d\in\mac{D}$                             & source document                                                                            & Section~\ref{sec:prelim} \\
        $q,\,a,\,e$                               & question, answer, evidence span ($e$ verbatim from $d$ or a snippet)                       & Section~\ref{sec:prelim} \\
        $h\in\{1,2,3,4\}$                         & prescribed hop count of the rollout                                                        & Section~\ref{subsec:verifier} \\
        $\pi_{t}^{\mathrm{pro}},\pi_{t}^{\mathrm{sol}}$ & proposer / solver policies at round $t$                                              & Section~\ref{sec:prelim} \\
        $\widetilde\pi_{\mathrm{sol},t}$          & solver under the single-turn search-disabled prompt                                       & Eq.~\eqref{eq:pplus-pminus} \\
        $\widetilde\pi_{\mathrm{aux},t}$          & auxiliary scorer (in our runs: weight-swapped solver)                                     & Eq.~\eqref{eq:pplus-pminus} \\
        $p_{+}^{t}(q,e,a),p_{-}^{t}(q,a)$         & with-evidence and no-evidence answer accuracies                                            & Eq.~\eqref{eq:pplus-pminus} \\
        $V_{t}(q,e,a)$                            & evidence-quality score $p_{+}^{t}-p_{-}^{t}$                                               & Eq.~\eqref{eq:vqe} \\
        $\widehat V_{t,m}$                        & Monte Carlo estimator of $V_{t}$ with $m$ samples                                         & Eq.~\eqref{eq:vhat-mc} \\
        $B(e)$                                    & brevity bonus, $L_{\max}=256$                                                              & Eq.~\eqref{eq:brevity} \\
        $F_{\mathrm{fmt}}$                        & 4-component proposer format reward                                                         & Eq.~\eqref{eq:fmt} \\
        $R^{\mathrm{DZ}}_{t}(q,a;k)$              & Dr.~Zero difficulty reward, $k\sim\mathrm{Bin}(n,p_{t}^{\mathrm{sol}})$                   & Eq.~\eqref{eq:dz-reward-raw} \\
        $\phi_{n}(p)$                             & population mean of $R^{\mathrm{DZ}}_{t}$                                                  & Eq.~\eqref{eq:dz-phi} \\
        $R^{\mathrm{pro}}_{t},R^{\mathrm{sol}}$   & proposer (Phase A) and solver (Phase B) rewards                                            & Eqs.~\eqref{eq:prop-reward},~\eqref{eq:sol-reward} \\
        $R_{\mathrm{correct}},R_{\mathrm{evidence}}$ & solver answer EM and evidence-F1 sub-rewards                                            & Eq.~\eqref{eq:sol-reward} \\
        $\lambda_{V},\lambda_{B},\lambda_{E}$     & EVE-Agent reward coefficients                                                              & Eqs.~\eqref{eq:prop-reward},~\eqref{eq:sol-reward} \\
        \midrule
        $z_{d}=\phi(d)\in\mab{R}^{D}$             & document embedding (E5-base-v2)                                                            & Section~\ref{subsec:selector} \\
        $K_{t}=K_{0}+\lfloor\alpha t\rfloor$      & cluster count schedule                                                                     & Section~\ref{subsec:selector} \\
        $\mac{C}_{t}=\{C_{t,k}\}_{k=1}^{K_{t}}$   & active clusters at round $t$                                                               & Section~\ref{subsec:selector} \\
        $N_{k},\,\overline R_{k}$                 & UCB1 pull count and empirical mean for arm $k$                                             & Appendix~\ref{app:selector-full} \\
        $\beta,\,\lambda_{u},\,\varepsilon$       & UCB exploration weight, utility weight, log smoothing                                      & Appendix~\ref{app:selector-full} \\
        \bottomrule
    \end{tabular}
\end{table}

\section{Theoretical results}\label{app:theory}

This appendix collects two basic properties of the reward design used throughout the paper. Lemma~\ref{lem:phi-n} characterizes the inherited difficulty reward as a unimodal function of the solver's per-trial success probability, and Proposition~\ref{prop:marginal} formalizes the evidence verifier as a marginal answer-accuracy gain with an unbiased, bounded-variance Monte Carlo estimator. Both results are referenced from the main text but were moved here to keep the body focused on the empirical contribution.

\subsection{Closed form and saturation of the difficulty reward}\label{app:theory-phi}

\begin{lemma}[Closed form and saturation of the difficulty reward]\label{lem:phi-n}
    Let $n\ge 2$ be the number of independent solver predictions, and let $p\in[0,1]$ be the probability that a single solver prediction matches the target answer. If $k\sim\mathrm{Bin}(n,p)$ denotes the number of correct predictions among the $n$ trials, then
    \begin{equation}
        \mathbb{E}_{k\sim\mathrm{Bin}(n,p)}
        \left[
            R^{\mathrm{DZ}}_{t}(q,a;k)
        \right]
        =
        \phi_{n}(p),
        \label{eq:lemma-dz-closed-form}
    \end{equation}
    where
    \begin{equation}
        \phi_{n}(p)
        \coloneqq
        \frac{n}{n-1}(1-p)\bigl(1-(1-p)^{n-1}\bigr).
        \label{eq:lemma-phi}
    \end{equation}
    The function $\phi_{n}:[0,1]\to[0,1]$ is continuous and unimodal, with $\phi_{n}(0)=\phi_{n}(1)=0$ and a unique interior maximizer
    \begin{equation}
        p^{\star}
        =
        1-n^{-1/(n-1)}.
        \label{eq:phi-maximizer}
    \end{equation}
\end{lemma}

\paragraph{Discussion.}
Lemma~\ref{lem:phi-n} makes explicit what the difficulty reward in Eq.~\eqref{eq:dz-reward-raw} encourages. The expected reward vanishes both when the solver almost never answers correctly and when it almost always does, and is largest at an intermediate success probability $p^{\star}$. This frontier is defined entirely by answer accuracy, however, and does not address whether the generated answer is supported by any source span---the motivation for the evidence verifier of Eq.~\eqref{eq:vqe}.

\begin{proof}[Proof of Lemma~\ref{lem:phi-n}]
Write $p_{t}^{\mathrm{sol}}(q,a)=p$ and let $k\sim\mathrm{Bin}(n,p)$. From Eq.~\eqref{eq:dz-reward-raw},
\begin{align*}
    \mab{E}\bigl[R^{\mathrm{DZ}}_{t}\bigr]
    &=
    \sum_{k=1}^{n-1}\binom{n}{k}p^{k}(1-p)^{n-k}\,\frac{n-k}{n-1} \\
    &=
    \frac{1}{n-1}\!\left[
        \sum_{k=0}^{n}\binom{n}{k}p^{k}(1-p)^{n-k}(n-k)
        -n(1-p)^{n}-0
    \right],
\end{align*}
where the $k=0$ term contributes $n(1-p)^{n}$ and the $k=n$ term contributes $0$. The full sum is the binomial mean of $n-k$, which equals $n(1-p)$, so the bracket equals
\begin{equation*}
    n(1-p)-n(1-p)^{n}
    =
    n(1-p)\bigl(1-(1-p)^{n-1}\bigr).
\end{equation*}
Dividing by $n-1$ gives $\phi_{n}(p)=\tfrac{n}{n-1}(1-p)(1-(1-p)^{n-1})$.

Continuity is immediate from the polynomial form. The boundary values are $\phi_{n}(0)=0$ and $\phi_{n}(1)=0$ by the factor $(1-p)$. Differentiating yields
$\phi_{n}'(p)=\tfrac{n}{n-1}\bigl[n(1-p)^{n-1}-1\bigr]$,
which vanishes exactly at $p^{\star}=1-n^{-1/(n-1)}\in(0,1)$. Since $\phi_{n}'(0)>0$ and $\phi_{n}'(1)<0$, this critical point is the unique interior maximum. Finally, $\phi_{n}(p)\in[0,1]$ for every $p\in[0,1]$ because $R^{\mathrm{DZ}}_{t}(q,a;k)\in[0,1]$ for every realization of $k$, so its expectation also lies in $[0,1]$.
\end{proof}

\subsection{Evidence verifier as marginal answer-accuracy gain}\label{app:theory-marginal}

\begin{proposition}[Evidence verifier as marginal answer-accuracy gain]\label{prop:marginal}
    For any valid triple $(q,a,e)$ and any training round $t$, the evidence verifier in Eq.~\eqref{eq:vqe} satisfies
    \begin{equation}
        V_{t}(q,e,a)
        =
        \mathbb{E}_{\hat a\sim\widetilde{\pi}_{\mathrm{sol},t}(\cdot\mid q,e)}
        \left[
            \mathbf{1}\{\hat a=a\}
        \right]
        -
        \mathbb{E}_{\hat a\sim\widetilde{\pi}_{\mathrm{aux},t}(\cdot\mid q)}
        \left[
            \mathbf{1}\{\hat a=a\}
        \right],
        \label{eq:prop-verifier}
    \end{equation}
    and therefore $V_t(q,e,a)\in[-1,1]$. When the auxiliary scorer shares its weights with the solver, the two distributions differ only in whether the evidence span $e$ is provided; in that case $V_t(q,e,a)$ is exactly the increase in the solver's answer probability induced by conditioning on $e$, and $V_t(q,e,a)=0$ whenever $\widetilde\pi_{\mathrm{sol},t}(\cdot\mid q,e)=\widetilde\pi_{\mathrm{sol},t}(\cdot\mid q)$ for every $e$.
    Moreover, the Monte Carlo estimator $\widehat V_{t,m}$ in Eq.~\eqref{eq:vhat-mc} is unbiased,
    \begin{equation}
        \mathbb{E}
        \left[
            \widehat V_{t,m}(q,e,a)
            \mid q,e,a
        \right]
        =
        V_t(q,e,a),
        \label{eq:vhat-unbiased}
    \end{equation}
    and its conditional variance satisfies
    \begin{equation}
        \mathrm{Var}
        \left[
            \widehat V_{t,m}(q,e,a)
            \mid q,e,a
        \right]
        \le
        \frac{1}{2m}.
        \label{eq:vhat-variance}
    \end{equation}
\end{proposition}

\paragraph{Discussion.}
Proposition~\ref{prop:marginal} clarifies the interpretation of the verifier introduced in Section~\ref{subsec:verifier}. A positive $V_t(q,e,a)$ means that the proposed evidence span makes the target answer more likely for the current solver, a value near zero means that the solver could already infer the answer from the question alone, and a negative value indicates a misleading or irrelevant span. The verifier thus turns evidence grounding into a reward-level quantity that can be optimized without oracle labels. We do not claim that this reward induces a non-vanishing policy gradient for every proposer parameterization; such a claim would require additional assumptions on the policy class and the optimization dynamics.

\begin{proof}[Proof of Proposition~\ref{prop:marginal}]
The identity~\eqref{eq:prop-verifier} is immediate from the definitions of $V_{t}$ and the two pointwise accuracies in Eqs.~\eqref{eq:pplus-pminus}--\eqref{eq:vqe}:
\begin{equation*}
    V_{t}(q,e,a)
    =
    \widetilde\pi_{\mathrm{sol},t}(a\mid q,e)
    -\widetilde\pi_{\mathrm{aux},t}(a\mid q),
\end{equation*}
which is the difference between two probabilities and hence lies in $[-1,1]$. In the weight-swap configuration $\widetilde\pi_{\mathrm{aux},t}\equiv\widetilde\pi_{\mathrm{sol},t}$, so $V_{t}(q,e,a)=\widetilde\pi_{\mathrm{sol},t}(a\mid q,e)-\widetilde\pi_{\mathrm{sol},t}(a\mid q)$. If $\widetilde\pi_{\mathrm{sol},t}(\cdot\mid q,e)=\widetilde\pi_{\mathrm{sol},t}(\cdot\mid q)$ for every $e$, then $V_{t}(q,e,a)=0$ identically.

For the Monte Carlo estimator, each indicator $\B{1}\{\hat a_{+}^{(j)}=a\}$ is a Bernoulli random variable with mean $\widetilde\pi_{\mathrm{sol},t}(a\mid q,e)$, and the analogous statement holds for $\B{1}\{\hat a_{-}^{(j)}=a\}$. By linearity of expectation, $\mab{E}[\widehat V_{t,m}\mid q,e,a]=V_{t}(q,e,a)$, establishing~\eqref{eq:vhat-unbiased}. The variance of a Bernoulli random variable with parameter $p$ is $p(1-p)\le 1/4$, so each of the two empirical means in Eq.~\eqref{eq:vhat-mc} has conditional variance at most $1/(4m)$. Because the two empirical means are computed from independent samples by construction of the rollout schedule,
\begin{equation*}
    \mathrm{Var}[\widehat V_{t,m}\mid q,e,a]
    \le
    \frac{1}{4m}+\frac{1}{4m}
    =
    \frac{1}{2m},
\end{equation*}
which proves~\eqref{eq:vhat-variance}.
\end{proof}

\section{Additional implementation details}\label{app:impl-details}

This appendix expands on the experimental setup of Section~\ref{subsec:setup}. The goal is to provide enough detail to reproduce all numbers in Tables~\ref{tab:bottleneck}--\ref{tab:main-joint}; nothing in this appendix changes any reward equation in the main text.

\paragraph{Backbone and tokenization.}
All four compared systems use the Qwen2.5-3B-Instruct backbone~\citep{qwen25}, the same tokenizer, and the same chat template. The auxiliary scorer in Eq.~\eqref{eq:vqe} reuses the current solver weights; we refer to this implementation as the \emph{weight-swap} configuration. The framework also supports a separately hosted frozen auxiliary scorer, but we do not use that configuration in any reported experiment.

\paragraph{Retrieval pipeline.}
The retrieval corpus is the FlashRAG Wikipedia-2018 snapshot ($\approx 21$M passages). Passages are encoded with E5-base-v2 using mean pooling and unit-norm $L^{2}$ normalization, and indexed with FAISS-IVF over $4{,}096$ centroids with $\mathrm{nprobe}=64$. The search tool returns the top-$3$ passages per query. Three CPU retrieval servers run in parallel to amortize tool latency.

\paragraph{Rollout protocol.}
A multi-turn proposer or solver rollout permits at most five assistant turns and at most $512$ tokens per tool response. The single-turn protocol used by the evidence verifier of Section~\ref{subsec:verifier} disables the search tool and limits the assistant to one turn; this is the configuration in which $p_{+}^{t}$ and $p_{-}^{t}$ of Eq.~\eqref{eq:pplus-pminus} are computed.

\paragraph{Data preparation.}
Proposer training prompts are drawn from the FlashRAG NQ--HotpotQA mixture, rebound to hop counts $h\in\{1,2,3,4\}$ in the ratio $4{:}3{:}2{:}1$. To construct the Phase B training set we roll out the frozen Phase A proposer on the same corpus with five samples per prompt and retain only valid triples; the proposer-emitted span $e$ becomes the gold evidence for the corresponding question--answer pair.

\paragraph{Optimization.}
Phase A runs $50$ HRPO steps with global batch size $256$, micro-batch $2$ per GPU, learning rate $1\mathrm{e}{-6}$, warmup ratio $0.03$, gradient clipping at $0.1$, the KL regularizer disabled, no nested grouping, and a verifier Monte Carlo budget $m=5$. Phase B runs $50$ GRPO steps with global batch size $256$, micro-batch $8$ per GPU, group size $5$, and the same optimizer settings as Phase A; the evidence-F1 coefficient is $\lambda_{E}=0.3$. Both phases use a single $8\times$B200 node.

\paragraph{Inference for the evaluation tables.}
For Tables~\ref{tab:bottleneck}--\ref{tab:main-joint}, each system is decoded greedily on each evaluation instance with the search tool enabled. The external judge for the evidence score is GPT-4.1, queried with the question, the gold answer, and the model-emitted span; the judge returns a binary decision over whether the question paired with the span is sufficient to recover the gold answer. The joint metric counts an instance as correct only when both the answer is exact-match equal to the gold answer (after normalization) and the judge labels the evidence as supporting.

\section{Format reward decomposition}\label{app:format}

The four sub-rewards of Eq.~\eqref{eq:fmt} are defined as follows.
Given a parsed proposer output with $h$ prescribed hops,
$T_{\mathrm{ass}}$ assistant turns, $T_{\mathrm{think}}$ assistant
turns that begin with a planning step, and $T_{\mathrm{tc}}$
syntactically valid tool-call emissions whose number matches the
number of returned tool responses,
\begin{align*}
    F_{\mathrm{think}}
    & = 
    \frac{T_{\mathrm{think}}}{\max(1,T_{\mathrm{ass}})},
    \\
    F_{\mathrm{tool}}(h)
    & = 
    \begin{cases}
        1 & h=1, \\
        \min\!\bigl(\tfrac{1+T_{\mathrm{tc}}}{h},\,1\bigr) & h>1,\,T_{\mathrm{tc}}=\text{\#returned responses},\\
        0 & h>1,\,\text{otherwise},
    \end{cases}
\end{align*}
and $F_{\mathrm{ans}}\in\{0,\tfrac{1}{2},1\}$ rewards a short,
in-context answer: writing $\tilde a\coloneqq\textsc{Norm}(a)$ and
letting $\Sigma$ denote the concatenation of the source document with
the returned tool responses,
\begin{equation*}
    F_{\mathrm{ans}}
     = 
    \begin{cases}
        1   & \tilde a\in\{\text{yes},\text{no}\}\text{ or }\bigl(\tilde a\subseteq\textsc{Norm}(\Sigma)\text{ and }|\tilde a|_{\text{words}}\le 5\bigr), \\
        \tfrac{1}{2} & \tilde a\subseteq\textsc{Norm}(\Sigma)\text{ and }5<|\tilde a|_{\text{words}}\le 10, \\
        0   & \text{otherwise.}
    \end{cases}
\end{equation*}
Rollouts with empty $q$ or empty $a$ receive $F_{\mathrm{fmt}}=0$ and
are filtered out before the verifier.

\section{Pseudocode}\label{app:pseudocode}

Algorithm~\ref{alg:phasea} describes one Phase A HRPO iteration with
the evidence verifier; Algorithm~\ref{alg:phaseb} describes one Phase
B GRPO iteration with the solver reward of Eq.~\eqref{eq:sol-reward}.
In Algorithm~\ref{alg:phasea}, the auxiliary scorer shares its weights
with the current solver and produces both the with-evidence and the
no-evidence single-turn rollouts.

\begin{algorithm}[h]
    \caption{One Phase A iteration of EVE-Agent (proposer)}\label{alg:phasea}
    \begin{algorithmic}[1]
        \State \textbf{Input:} proposer $\pi^{\mathrm{pro}}_{t}$, solver $\pi^{\mathrm{sol}}$ (frozen), corpus $\mac{D}$, hop pmf, coefficients $(\lambda_{V},\lambda_{B})$, MC budget $m$, group size $n$
        \State Sample $\{d_{i}\}_{i=1}^{N}\overset{\text{iid}}{\sim}\mathrm{Unif}(\mac{D})$ and $h_{i}\sim\text{hop pmf}$
        \For{$i=1,\dots,N$}
            \State $(q_{i},a_{i},e_{i})\sim\pi^{\mathrm{pro}}_{t}(\cdot\mid d_{i},h_{i},\mac{R})$
            \State $F_{\mathrm{fmt},i}\gets\textsc{ComputeFormatScore}(q_{i},a_{i},e_{i},d_{i},h_{i})$
            \State Mark $i$ \emph{invalid} if parse fails or $F_{\mathrm{fmt},i}=0$
        \EndFor
        \For{each \emph{valid} $i$}
            \State Sample $\{\hat a_{j}^{i}\}_{j=1}^{n}\sim\pi^{\mathrm{sol}}(\cdot\mid q_{i},\mac{R})$ \Comment{multi-turn, with search}
            \State $k_{i}\gets\sum_{j=1}^{n}\B{1}\{\hat a_{j}^{i}=a_{i}\}$
            \State $R^{\mathrm{DZ}}_{i}\gets\B{1}\{0<k_{i}<n\}(n-k_{i})/(n-1)$
            \State Sample $\{\hat a_{+}^{(j),i}\}_{j=1}^{m}\sim\widetilde\pi^{\mathrm{sol}}(\cdot\mid q_{i},e_{i})$ \label{alg:line:vplus} \Comment{single-turn, no search}
            \State Sample $\{\hat a_{-}^{(j),i}\}_{j=1}^{m}\sim\widetilde\pi^{\mathrm{aux}}(\cdot\mid q_{i})$ \Comment{shares weights with the solver}
            \State $\widehat V_{i}\gets\tfrac{1}{m}\sum_{j}\B{1}\{\hat a_{+}^{(j),i}=a_{i}\}-\tfrac{1}{m}\sum_{j}\B{1}\{\hat a_{-}^{(j),i}=a_{i}\}$
            \State $B_{i}\gets\max(0,\,1-|e_{i}|/L_{\max})$
            \State $R^{\mathrm{pro}}_{i}\gets\tfrac{1}{2}F_{\mathrm{fmt},i}+R^{\mathrm{DZ}}_{i}+\lambda_{V}\widehat V_{i}+\lambda_{B}B_{i}$
        \EndFor
        \For{invalid $i$}
            \State $R^{\mathrm{pro}}_{i}\gets\tfrac{1}{2}F_{\mathrm{fmt},i}$
        \EndFor
        \State Compute hop-grouped advantages $A_{i,h}$ by Eq.~\eqref{eq:hrpo-adv}
        \State Update $\pi^{\mathrm{pro}}$ with one HRPO step on $\{(q_{i},a_{i},e_{i}),A_{i,h}\}$
    \end{algorithmic}
\end{algorithm}

\begin{algorithm}[h]
    \caption{One Phase B iteration of EVE-Agent (solver)}\label{alg:phaseb}
    \begin{algorithmic}[1]
        \State \textbf{Input:} Phase A proposer checkpoint (frozen), solver $\pi^{\mathrm{sol}}_{t}$, training shard $\{(q_{i},a_{i},e_{i})\}_{i=1}^{N}$ generated by the Phase A proposer, GRPO group size $n$, coefficient $\lambda_{E}$
        \For{$i=1,\dots,N$}
            \State Sample $\{(\hat a_{j}^{i},\hat e_{j}^{i})\}_{j=1}^{n}\sim\pi^{\mathrm{sol}}_{t}(\cdot\mid q_{i},\mac{R})$
            \For{$j=1,\dots,n$}
                \State $R_{\mathrm{correct},j}^{i}\gets\B{1}\{\textsc{EM}(\hat a_{j}^{i},a_{i})\}$
                \State $R_{\mathrm{evidence},j}^{i}\gets\mathrm{F1}_{\mathrm{tok}}(\textsc{Norm}(\hat e_{j}^{i}),\textsc{Norm}(e_{i}))$
                \State $R^{\mathrm{sol},i}_{j}\gets R_{\mathrm{correct},j}^{i}+\lambda_{E}R_{\mathrm{evidence},j}^{i}$
            \EndFor
        \EndFor
        \State Update $\pi^{\mathrm{sol}}$ with one GRPO step of Eq.~\eqref{eq:grpo} on $\{R^{\mathrm{sol},i}_{j}\}$
    \end{algorithmic}
\end{algorithm}

\section{Selector: full specification}\label{app:selector-full}

This appendix gives the complete specification of the cluster-bandit
selector introduced in Section~\ref{subsec:selector}.

\paragraph{Document embeddings.}
A frozen sentence encoder $\phi$ (E5-base-v2 with mean pooling and
$L^{2}$ normalization) maps each $d\in\mac{D}$ to
$z_{d}=\phi(d)\in\mab{R}^{D}$, $D=768$. Embeddings are precomputed
once per corpus.

\paragraph{Adaptive clustering.}
At round $t$, the cluster set $\mac{C}_{t}=\{C_{t,k}\}_{k=1}^{K_{t}}$
is obtained by mini-batch $k$-means on $\{z_{d}\}_{d\in\mac{D}}$ with
a granularity schedule
\begin{equation}
    K_{t}
     = 
    K_{0}+\lfloor\alpha\,t\rfloor,
    \qquad K_{0}=10, \alpha\ge 0.
    \label{eq:selector-K}
\end{equation}
With $\alpha=0$ the clustering is fixed throughout training; with
$\alpha>0$, whenever $K_{t}>K_{t-1}$ we re-cluster the corpus and
\emph{inherit} the bandit statistics across the split. Concretely,
each new centroid is assigned to its nearest old centroid by
Euclidean distance, and if old cluster $j$ has $n_{\mathrm{children}}$
new children with counts $N_{j}$ and reward sum $S_{j}$, each child
starts with $\max(\lfloor N_{j}/n_{\mathrm{children}}\rfloor,1)$ pulls
and $S_{j}/n_{\mathrm{children}}$ reward. Parent statistics are split
across children rather than discarded; in particular, we do not reset
arm statistics on split.

\paragraph{UCB1 bandits.}
Two independent UCB1 bandits run side by side; both initialize every
arm with one virtual pull and zero virtual reward so the exploration
bonus is finite from round $1$. For arm $k$ with $N_{k}$ pulls,
reward sum $S_{k}$, and total pulls $N_{\mathrm{tot}}=\sum_{j}N_{j}$,
the UCB score is
\begin{equation}
    U_{k}
     = 
    \frac{S_{k}}{\max(N_{k},1)}
     + 
    \beta\sqrt{\frac{\log\max(N_{\mathrm{tot}},1)}{\max(N_{k},1)}},
    \qquad \beta>0,
    \label{eq:selector-ucb}
\end{equation}
with $\beta=1$ throughout. When a single arm is required, we use
the deterministic $\argmax_{k}U_{k}$; when a batch of $n>1$ arms is
required, the batch is drawn from a softmax over $U_{k}$ rescaled by
the empirical standard deviation of the $U$-values, which keeps the
batch diverse without sacrificing the exploration bias of UCB1.

\paragraph{Cluster and task arm sets.}
The cluster bandit has $K_{t}$ arms; arm $k$ corresponds to the
cluster $C_{t,k}$. The task-type bandit has five arms with labels
\textsc{factual}, \textsc{comparison}, \textsc{causal},
\textsc{temporal}, and \textsc{aggregation}. The hop count is already
controlled independently through the hop pmf and is therefore not a
task type.

\paragraph{Within-cluster sampling.}
Let $u_{d}^{t}$ be the number of times document $d$ has been previously
sampled at round $t$ and let $\delta_{d,k}=\Vert z_{d}-\mu_{k}\Vert_{2}$
be its Euclidean distance to the cluster centroid $\mu_{k}$. Inside
the chosen cluster $C_{t,k_{t}}$, a document is drawn with probability
\begin{equation}
    \mab{P}\bigl[d\mid k_{t}\bigr]
    \propto
    \delta_{d,k_{t}}\cdot\frac{1}{1+u_{d}^{t}},
    \qquad d\in C_{t,k_{t}}.
    \label{eq:selector-sample}
\end{equation}
The first factor biases sampling toward the cluster boundary, where
atypical entities concentrate; the second is a soft
sampling-without-replacement that discourages revisiting
heavily-used documents inside the same cluster. If every score
collapses to zero (e.g.\ for an empty cluster), the policy falls
back to a uniform draw.

\paragraph{Between-iteration reward feedback.}
The selector is updated between self-evolution iterations rather than
within a single policy-gradient step. After each iteration, we read
(i) the generated solver-training samples, annotated with the
selector's chosen cluster id and task-type id per row, and (ii) the
run-level summary score $R_{\mathrm{run}}\in[0,1]$ of the
just-finished iteration (the average of the latest proposer and
solver mean rewards). The per-sample reward is the product of
$R_{\mathrm{run}}$ and a bounded quality proxy $Q_{i}\in[0,1]$
defined for sample $i$ as
\begin{equation}
    Q_{i}
     = 
    q_{\mathrm{len}}(|q_{i}|)\cdot q_{\mathrm{ans}}(|a_{i}|)\cdot
    \frac{1}{\sqrt{\max(D_{i},1)}},
\end{equation}
where $q_{\mathrm{len}}(\ell)=1$ if $20\le\ell\le 220$ and $0.65$
otherwise, $q_{\mathrm{ans}}(\ell)=1$ if $1\le\ell\le 80$ and $0.55$
otherwise, and $D_{i}\ge 1$ is the number of duplicates of the
prompt string in the batch. Let
$r_{i}\coloneqq R_{\mathrm{run}}Q_{i}$ be the resulting per-sample
reward. Conditional on the cluster id $k(i)$, the selector reward fed
back to the bandit is
\begin{equation}
    R^{\mathrm{sel}}_{i}
     = 
    \underbrace{-\log\bigl(N_{k(i)}/N_{\mathrm{tot}}+\varepsilon\bigr)}_{=:R_{\mathrm{div}}(k(i))}
     + 
    \lambda_{u}\,
    \underbrace{\bigl(r_{i}-\overline r_{k(i)}^{\mathrm{prev}}\bigr)}_{=:R_{\mathrm{util}}(i)},
    \label{eq:selector-reward}
\end{equation}
where $N_{k},N_{\mathrm{tot}}$ are the cluster bandit's pre-update
pull counts and $\overline r_{k}^{\mathrm{prev}}$ is its mean reward
\emph{before} the current update. We use
$(\lambda_{u},\varepsilon)=(0.5,10^{-8})$. The same reward formula is
applied to the task-type bandit by replacing the cluster id $k(i)$
with the task-type id $\tau(i)$.

\section{Hyperparameters}\label{app:hp}

Table~\ref{tab:hp} lists the hyperparameter values used throughout.

\begin{table}[h]
    \centering
    \caption{EVE-Agent hyperparameters.}
    \label{tab:hp}
    \footnotesize
    \begin{tabular}{l l}
        \toprule
        Name                                & Value \\
        \midrule
        $\lambda_{V}$ (evidence-quality coefficient)    & $0.5$ \\
        $\lambda_{B}$ (brevity coefficient)             & $0.1$ \\
        $L_{\max}$ (brevity cutoff)                     & $256$ tokens \\
        $m$ (verifier Monte Carlo budget)               & $5$ \\
        $\lambda_{E}$ (solver evidence-F1 coefficient)  & $0.3$ \\
        Phase A HRPO steps                              & $50$ \\
        Phase B GRPO steps                              & $50$ \\
        Global batch size                               & $256$ \\
        Hop pmf $\{1,2,3,4\}$                           & $4{:}3{:}2{:}1$ \\
        \midrule
        $K_{0}$ (initial cluster count)                 & $10$ \\
        $\alpha$ (granularity slope)                    & $0.0$ default \\
        $\beta$ (UCB exploration weight)                & $1.0$ \\
        $\lambda_{u}$ (selector utility weight)         & $0.5$ \\
        $\varepsilon$ (diversity smoothing)             & $10^{-8}$ \\
        \bottomrule
    \end{tabular}
\end{table}

\section{Extended evidence-grounding diagnostics}\label{app:bottleneck-full}

Table~\ref{tab:bottleneck-full} extends Table~\ref{tab:bottleneck} of Section~\ref{subsec:exp-bottleneck} by adding the per-benchmark answer-accuracy and evidence-presence columns. The evaluation uses the same $1{,}000$ instances per benchmark, with $125$ instances on Bamboogle.

\begin{table}[h]
    \centering
    \caption{Full evidence-quality breakdown on $1{,}000$ samples per
        benchmark ($125$ for Bamboogle). \emph{Initial} is the
        Qwen2.5-3B-Instruct backbone; \emph{Prior} is a faithful
        re-implementation of~\citet{yue2026drzero}; \emph{EVE-Agent}
        is the verifier-trained solver.}
    \label{tab:bottleneck-full}
    \setlength{\tabcolsep}{3.5pt}
    \footnotesize
    \begin{tabular}{l c c c c}
        \toprule
        Benchmark & Initial (no s.) & Initial (s.) & Prior & EVE-Agent \\
        \midrule
        \multicolumn{5}{l}{\emph{Answer accuracy}} \\
        NQ        & $0.068$ & $0.032$ & $0.069$ & $\mathbf{0.289}$ \\
        TriviaQA  & $0.219$ & $0.125$ & $0.257$ & $\mathbf{0.437}$ \\
        PopQA     & $0.081$ & $0.055$ & $0.134$ & $\mathbf{0.300}$ \\
        HotpotQA  & $0.057$ & $0.029$ & $0.110$ & $\mathbf{0.209}$ \\
        2Wiki     & $0.034$ & $0.013$ & $0.077$ & $\mathbf{0.176}$ \\
        MuSiQue   & $0.015$ & $0.010$ & $\mathbf{0.055}$ & $0.050$ \\
        Bamboogle & $\mathbf{0.200}$ & $0.024$ & $0.104$ & $0.088$ \\
        \midrule
        \multicolumn{5}{l}{\emph{Evidence-present rate} (fraction of rollouts emitting an evidence span)} \\
        NQ        & $0.996$ & $0.785$ & $0.991$ & $0.999$ \\
        TriviaQA  & $0.996$ & $0.723$ & $0.977$ & $0.991$ \\
        HotpotQA  & $0.996$ & $0.567$ & $0.971$ & $0.998$ \\
        2Wiki     & $0.996$ & $0.555$ & $0.933$ & $0.997$ \\
        MuSiQue   & $0.993$ & $0.585$ & $0.911$ & $0.996$ \\
        \bottomrule
    \end{tabular}
\end{table}

Two patterns confirm the bottleneck argument of
Section~\ref{subsec:exp-bottleneck}. First, the prior system emits an
evidence span in $90$--$99\%$ of cases, comparable to the initial
backbone, yet its judged evidence score in Table~\ref{tab:bottleneck}
is low. Second, the answer-and-evidence gap is much larger than the
answer-only gap on the open-domain datasets: a verifier-trained solver
is several times more likely to be simultaneously correct \emph{and}
supported.

\end{document}